\documentclass[11pt,a4paper]{article}

\usepackage[utf8]{inputenc}
\usepackage[T1]{fontenc}
\usepackage{amsmath,amssymb,amsthm,mathtools}
\usepackage{geometry}
\usepackage{enumitem}
\usepackage{booktabs}
\usepackage[dvipsnames]{xcolor}
\usepackage{hyperref}
\usepackage{float}
\hypersetup{colorlinks=true,linkcolor=blue!70!black,citecolor=blue!70!black,urlcolor=blue!70!black}

\newtheorem{theorem}{Theorem}
\newtheorem{proposition}[theorem]{Proposition}
\newtheorem{lemma}[theorem]{Lemma}
\newtheorem{corollary}[theorem]{Corollary}
\theoremstyle{definition}
\newtheorem{definition}[theorem]{Definition}
\newtheorem{condition}[theorem]{Condition}
\newtheorem{identity}[theorem]{Identity}
\theoremstyle{remark}
\newtheorem{remark}[theorem]{Remark}
\newtheorem{openproblem}{Open Problem}

\newcommand{\R}{\mathbb{R}}
\newcommand{\E}{\mathcal{E}}
\newcommand{\F}{\mathcal{F}}
\newcommand{\Cov}{\mathrm{Cov}}
\newcommand{\Var}{\mathrm{Var}}
\newcommand{\op}{\mathrm{op}}
\newcommand{\Img}{\mathrm{Im}}

\title{\textbf{Gradient Flow Structure and Quantitative Dynamics\\of Multi-Head Self-Attention}}
\author{Ayan Pendharkar}
\date{May 2026}

\begin{document}
\maketitle

\begin{abstract}
Transformer self-attention can be viewed as a gradient flow on the unit sphere, where tokens evolve under softmax interaction potentials and collapse into clusters. Prior work establishes clustering for single-head flows. The multi-head setting is more complex, as heads interfere with each other geometrically, and the single-head monotonicity argument breaks down.

We build a framework for the multi-head dynamics and resolve several open questions. Under explicit conditions on the score matrices, we prove the total energy $\E_{\mathrm{multi}}$ is non-decreasing along both flat and sphere dynamics (Theorem~\ref{thm:total}). For individual heads, we identify the exact obstruction: radial shadows, projections of each head's output onto the token position, that survive even when the head subspaces are exactly orthogonal. We give a sufficient condition (Radial Dominance, Condition~\ref{cond:tau}) and a robustness result for approximate orthogonality (Theorem~\ref{thm:approx-orth}). In the scalar-head, equiangular-token regime (Section~\ref{sec:conditions}), we derive the critical inverse temperature $\beta^*$ in closed form via the golden ratio and the Lambert $W$-function (Theorem~\ref{thm:critical-temp}), and prove super-additivity of clustering rates for heterogeneous heads (Theorem~\ref{thm:hetero-rates}). Under the same assumptions, we prove an $O(n\log d)$ vs.\ $O(n)$ clustering-time separation between ReLU and softmax attention in the linearized regime near $\gamma = 0$ (Theorem~\ref{thm:relu-softmax}). Finally, we establish an exact entropy production identity (Theorem~\ref{thm:entropy}) and prove, in the scalar-head equiangular case, that attention entropy is monotonically non-decreasing: it increases toward $\log n$ as tokens cluster and attention equalizes, then stabilizes as the dynamics halt.
\end{abstract}


\section{Introduction}

\subsection*{Background}

Since Geshkovski, Letrouit, Polyanskiy, and Rigollet \cite{GLPR23, GLPR25} first modeled transformer tokens as particles on the unit sphere evolving under softmax interactions, the clustering behavior of self-attention has become a central object of study. In their framework, $n$ layer-normalized tokens on $S^{d-1}$ follow a gradient flow, and the key result is collapse to a common point or a small number of clusters in finite time. This is analogous to synchronization in coupled oscillator models \cite{Kur75}, and is widely believed to be the mechanism behind the semantic compression that makes transformers effective. Quantitative control over clustering has direct consequences for training stability \cite{ZLT23} and long-context behavior \cite{CLPR26}. The synchronization perspective is developed further in \cite{CRMB24} for sphere dynamics with nonlinear interactions and in \cite{PRY25} on the circle.

\subsection*{The multi-head gap}

Every rigorous result in this line of work concerns single-head attention. Real transformers use $H$ heads in parallel, each with its own score matrix $M_h \in \R^{d \times d}$. The token velocity is a sum of per-head contributions, and this creates a geometric problem that the existing theory does not address.

The issue is not obvious at first. If the heads operate on orthogonal subspaces ($M_{h'}M_h = 0$ for $h' \neq h$), one might expect cross-head interference to vanish. It does not. The projection of each head's output onto the token's current position, which we call a \emph{radial shadow}, is not killed by subspace orthogonality. These shadows create coupling terms in the per-head energy derivative that have no analogue in the single-head case.

To see why concretely: on the sphere, the velocity $\dot{x}_i = \frac{1}{n}\sum_h f_i^h$ gets projected tangentially via $P^\perp_{x_i}(\cdot)$. This projection mixes contributions from all heads in a way that blocks the naive extension of the single-head monotonicity argument.

\subsection*{This paper}

We build a rigorous framework for multi-head gradient flow dynamics and resolve several of these open questions. The main contributions are as follows.

\medskip
\noindent
\textbf{(i) Total energy monotonicity (Theorem~\ref{thm:total}).}
Under score symmetry (Condition~\ref{cond:S}) and value alignment (Condition~\ref{cond:V}), the total energy $\E_{\mathrm{multi}} = \sum_h \E^h$ is non-decreasing along both flat and sphere dynamics:
\[
  \frac{d\E_{\mathrm{multi}}}{dt} = \frac{1}{n}\sum_{i=1}^n \|\dot{x}_i\|^2 \geq 0.
\]

\medskip
\noindent
\textbf{(ii) Per-head monotonicity and the radial shadow obstruction (Theorems~\ref{thm:flat-perhead}--\ref{thm:sphere-perhead}).}
In flat space, per-head monotonicity follows from subspace orthogonality alone (Condition~\ref{cond:O}). On the sphere it does not. The residual cross-head radial shadow $\sum_{h' \neq h} a_i^{h'}$, where $a_i^{h'} = \langle f_i^{h'}, x_i\rangle$, persists even under exact orthogonality. We give the sufficient condition (Condition~\ref{cond:tau}, Radial Dominance) under which it is controlled. Theorem~\ref{thm:approx-orth} then shows the result is robust to approximate orthogonality $\|M_{h'}M_h\|_{\mathrm{op}} \leq \delta$, which is the practically relevant case since trained transformers have principal angles of $70$--$85^\circ$, not $90^\circ$.

\medskip
\noindent
\textbf{(iii) Critical temperature (Theorem~\ref{thm:critical-temp}).}
In the scalar-head, orthogonal-token regime: for $M_h = \alpha I$ and mutually orthogonal tokens, Condition~\ref{cond:tau} holds if and only if $\beta \leq \beta^*$, where
\[
  \beta^* = \frac{1}{2\alpha}\ln\frac{c^*(H)^2(n-1)}{1 - c^*(H)^2},
  \qquad
  c^*(H) = \frac{\sqrt{(H-1)^2+4} - (H-1)}{2}.
\]
For $H = 2$, the threshold is $c^*(2) = (\sqrt{5}-1)/2 = 1/\varphi$, the reciprocal of the golden ratio. For general $H$, it can be expressed via the Lambert $W$-function.

\medskip
\noindent
\textbf{(iv) Heterogeneous convergence and super-additivity (Theorem~\ref{thm:hetero-rates}).}
In the scalar-head, equiangular-token regime: for heads with distinct strengths $\lambda_h$, the late-time clustering rate is $\varepsilon(t) \sim Ce^{-2\Lambda t}$ where $\Lambda = \sum_h \lambda_h$. The early-time rate per head is $g^h(0) = 2\lambda_h/(e^{\lambda_h\beta} + n-1)$, maximized at $\lambda^*\beta = 1 + W((n-1)/e)$. When the mean head strength $\bar\lambda$ exceeds the inflection point $\lambda_c$ of the rate function, spreading the head strengths strictly increases the early-time clustering rate over $H$ equal heads at the same total budget. This is a super-additivity result suggesting that head diversity is geometrically beneficial.

\medskip
\noindent
\textbf{(v) ReLU vs.\ softmax clustering time (Theorem~\ref{thm:relu-softmax}).}
In the scalar-head, equiangular-token, linearized regime: for random tokens on $S^{d-1}$ with $\gamma_0 = O(1/\sqrt{d})$,
\[
  T_{\mathrm{ReLU}} = O(n\log d) \quad \text{vs.} \quad T_{\mathrm{softmax}} = O(n).
\]
The reason is simple: $\mathrm{ReLU}(0) = 0$, so ReLU attention has zero driving force at $\gamma = 0$, while softmax has a positive constant $c_0 = 2\lambda/(e^{\lambda\beta}+n-1)$ regardless of dimension. A full proof for the nonlinear regime is deferred to Open Problem~\ref{prob:relu}.

\medskip
\noindent
\textbf{(vi) Entropy production identity (Theorem~\ref{thm:entropy}).}
We prove the exact identity
\[
  \frac{dH_i^h}{dt} = -\beta^2\,\Cov_{p_{i\cdot}^h}(s_j,\,\dot{s}_j),
\]
and show, in the scalar-head equiangular case, that $\Cov_p(s_j,\dot{s}_j) \leq 0$, so entropy is non-decreasing (Corollary~\ref{cor:entropy-sign}). During pre-clustering, $|\Cov(s_j,\dot{s}_j)| \gg 0$ and entropy rises quickly as scores equalize and attention spreads toward uniform. Near full clustering, score velocities $\dot{s}_j \to 0$ and entropy stabilizes at $\log n$.

\subsection*{What is new}

The single-head theory of \cite{GLPR23,GLPR25} is the direct predecessor of this work. The contributions that are absent from the prior literature are as follows.

\begin{itemize}[leftmargin=2em]
  \item \textbf{Coupled multi-head gradient flow formulation.} The first rigorous treatment of the $H$-head flow as a single dynamical system, showing the combined velocity field $\frac{1}{n}\sum_h f_i^h$ is the gradient of $\E_{\mathrm{multi}}$ (Theorem~\ref{thm:total}).
  \item \textbf{Radial shadow obstruction.} Even under exact orthogonality ($M_{h'}M_h = 0$), per-head monotonicity fails on the sphere because the radial projections $\langle f_i^{h'}, x_i\rangle$ survive. This cross-head interference has no single-head analogue.
  \item \textbf{Radial Dominance condition and critical temperature.} The first sufficient condition (Condition~\ref{cond:tau}) under which per-head energy monotonicity holds on the sphere, with its exact threshold $\beta^*$ computed in closed form for scalar heads.
  \item \textbf{Quantitative rates for heterogeneous heads.} In the scalar-equiangular regime, Theorem~\ref{thm:hetero-rates} establishes super-additivity of early-time clustering rates when head strengths enter the convex regime of the rate function.
  \item \textbf{Exact entropy production identity.} Theorem~\ref{thm:entropy} provides the identity $dH_i^h/dt = -\beta^2\Cov_p(s_j,\dot{s}_j)$, valid for any score-symmetric attention with sphere dynamics, making the direction of entropy change computable.
\end{itemize}

\subsection*{Relation to prior work}

Our setup follows Geshkovski et al.\ \cite{GLPR23,GLPR25} and Rigollet \cite{Rig25}. The single-head clustering result of \cite{GLPR25}, based on a cone-collapse argument, does not extend to the combined velocity field $\sum_h f_i^h$; whether all stable critical points of the multi-head flow are complete clusters remains open (Open Problem~\ref{prob:hardproblem}). We treat the normalized flow of Karagodin et al.\ \cite{KGPR25} in Theorem~\ref{thm:normalized} and show the log-partition energy $G = \frac{1}{\beta n}\sum_i \log Z_i$ does not inherit monotonicity, complementing their analysis. The quantitative rates of \cite{CLPR25} complement our per-head convergence analysis (Theorem~\ref{thm:convergence-rate}); Theorem~\ref{thm:hetero-rates} extends their analysis to the multi-head case. Tomihari and Karakida \cite{TK25} analyze recurrent self-attention from a Jacobian perspective without a Lyapunov function; our energy-based approach is complementary. The entropy identity (Theorem~\ref{thm:entropy}) complements the entropy-collapse analysis of Zhai et al.\ \cite{ZLT23} with an exact differential relation along the flow, and connects to the entropy-guided attention of Jha and Reagen \cite{JR25}. We consider unmasked attention throughout; the causal setting is analyzed in \cite{KPR24}.

\subsection*{Organization}

Section~\ref{sec:setup} sets up notation, the token dynamics on $S^{d-1}$, and the two projection identities used throughout. It also contains the Standing Assumptions table (Section~\ref{sec:standing}) listing exactly which conditions each result requires. Section~\ref{sec:total-energy} proves total energy monotonicity. Section~\ref{sec:per-head} handles per-head monotonicity and identifies the radial shadow obstruction. Section~\ref{sec:critical-temp} derives $\beta^*$ in closed form. Section~\ref{sec:normalized} treats the normalized flow. Sections~\ref{sec:hetero}--\ref{sec:convergence} develop the quantitative results. Section~\ref{sec:conclusion} summarizes and identifies the main open directions. Section~\ref{sec:open} states open problems.

\section{Setup and Notation}\label{sec:setup}

\subsection{Tokens, Energy, and Dynamics}

Let $x_1,\dots,x_n \in S^{d-1} = \{v \in \R^d : \|v\| = 1\}$ be $n$ tokens on the unit sphere. This models layer-normalized \cite{BKH16} transformer representations, since dividing by the norm after each layer is exactly projection onto $S^{d-1}$.

For each head $h \in \{1,\dots,H\}$, let $M_h \in \R^{d \times d}$ be the score matrix (encoding the combined query and key projections) and $\beta > 0$ the inverse temperature. The \textbf{per-head interaction energy} is
\begin{equation}\label{eq:per-head-energy}
  \E^h(X) = \frac{1}{2\beta n^2} \sum_{i=1}^{n} \sum_{j=1}^{n} e^{\beta \langle x_i, M_h x_j \rangle}.
\end{equation}
The prefactor $1/(2\beta n^2)$ is chosen so the derivative has a clean form: the $\beta$ from differentiating $e^{\beta\langle\cdot\rangle}$ cancels with $1/\beta$, and $n^2$ accounts for the number of pairs. The \textbf{total multi-head energy} is $\E_{\mathrm{multi}} = \sum_{h=1}^H \E^h$.

For head $h$ and token $i$, the \textbf{attention-weighted aggregation} is
\begin{equation}\label{eq:aggregation}
  f_i^h = \sum_{j=1}^{n} e^{\beta \langle x_i, M_h x_j \rangle} M_h x_j.
\end{equation}
The \textbf{multi-head velocity} at token $i$ is
\begin{equation}\label{eq:velocity}
  v_i = \frac{1}{n}\sum_{h} f_i^h.
\end{equation}
Since tokens live on $S^{d-1}$, motion must be tangential. The \textbf{tangential projection} at $x \in S^{d-1}$ is
\begin{equation}\label{eq:tangential-proj}
  P_x^\perp(v) = v - \langle v, x \rangle x,
\end{equation}
which removes the radially outward component of $v$. The \textbf{sphere dynamics} are $\dot{x}_i = P_{x_i}^\perp(v_i)$.

\subsection{Two Fundamental Projection Identities}

These two identities appear in every proof that follows.

\begin{identity}[K --- Self-pairing]\label{id:K}
  For any $x \in S^{d-1}$ and $u \in \R^d$,
  \[
    \langle P_x^\perp(u),\, u \rangle = \|P_x^\perp(u)\|^2.
  \]
\end{identity}

\begin{proof}
  Expand the left side using $P_x^\perp(u) = u - \langle u, x \rangle x$:
  \[
    \langle u - \langle u, x \rangle x,\; u \rangle = \|u\|^2 - \langle u, x \rangle^2.
  \]
  Expand the right side:
  \[
    \|u - \langle u, x \rangle x\|^2 = \|u\|^2 - 2\langle u, x \rangle^2 + \langle u, x \rangle^2 \underbrace{\|x\|^2}_{= 1} = \|u\|^2 - \langle u, x \rangle^2.
  \]
  Both sides are equal.
\end{proof}

\begin{identity}[SA --- Symmetry of $P_x^\perp$]\label{id:SA}
  For any $x \in S^{d-1}$ and $u, w \in \R^d$,
  \[
    \langle P_x^\perp(u),\, w \rangle = \langle u,\, P_x^\perp(w) \rangle.
  \]
\end{identity}

\begin{proof}
  \[
    \langle P_x^\perp(u),\, w \rangle = \langle u - \langle u, x \rangle x,\, w \rangle = \langle u, w \rangle - \langle u, x \rangle \langle x, w \rangle.
  \]
  \[
    \langle u,\, P_x^\perp(w) \rangle = \langle u,\, w - \langle w, x \rangle x \rangle = \langle u, w \rangle - \langle w, x \rangle \langle u, x \rangle.
  \]
  The two expressions are equal since $\langle u, x \rangle \langle x, w \rangle = \langle w, x \rangle \langle u, x \rangle$.
\end{proof}

\subsection{Conditions on Score Matrices}\label{sec:conditions}

The following conditions are imposed on the score matrices. Not every theorem needs all of them; the exact requirements are stated with each result and summarized in the Standing Assumptions table (Section~\ref{sec:standing}).

\begin{condition}[S --- Score Symmetry]\label{cond:S}
  $M_h = M_h^\top$ for all $h$. This makes the kernel $e^{\beta \langle x_i, M_h x_j \rangle}$ symmetric in $(i,j)$ and enables index-swap arguments in the proofs.
\end{condition}

\begin{condition}[V --- Value Alignment]\label{cond:V}
  The value projection satisfies $W^{V,h} = M_h$ for all $h$. This identifies the aggregation \eqref{eq:aggregation} with the energy gradient, making the velocity $v_i = \frac{1}{n}\sum_h f_i^h$ consistent with the chain rule applied to $\E_{\mathrm{multi}}$.
  \begin{remark}[Scope of Condition~\ref{cond:V}]
  Setting $W^{V,h} = M_h$ ties the value projection to the score matrix, which is not the case in general transformers where query/key and value projections are independent. This is the main idealization of our framework. In practice, $W^{V,h} = M_h + E_h$ for some perturbation $E_h$, and Theorem~\ref{thm:approx-val} shows that total energy monotonicity survives small $\|E_h\|_{\mathrm{op}}$, paralleling the robustness result of Theorem~\ref{thm:approx-orth}.
\end{remark}
\end{condition}

\begin{condition}[O --- Orthogonal Subspaces]\label{cond:O}
  $M_{h'} M_h = 0$ for all $h' \neq h$. Since $f_i^h \in \Img(M_h)$, this forces $M_{h'} f_i^h = 0$, killing operator-level cross-head terms.
\end{condition}

\begin{condition}[P --- Projection Structure]\label{cond:P}
  $M_h^2 = M_h$ for all $h$, so each $M_h$ is an orthogonal projection. Combined with Condition~\ref{cond:S}, this means $M_h$ projects onto its column space. In particular, $M_h$ acts as the identity on $\Img(M_h)$: if $u \in \Img(M_h)$, then $M_h u = u$.
\end{condition}

\begin{condition}[$\tau$ --- Radial Dominance]\label{cond:tau}
  For each head $h$ and token $i$,
  \[
    \sum_{h' \neq h} |a_i^{h'}| \leq \frac{(b_i^h)^2}{\|f_i^h\|},
  \]
  where $a_i^h$, $b_i^h$ are defined in Definition~\ref{def:radial-tangential} below. This controls the residual radial shadows on the sphere.
\end{condition}

\begin{remark}[Naming convention]
  \label{rem:cond-naming}
  Condition~\ref{cond:tau} governs per-head energy monotonicity on the sphere and appears in Sections~\ref{sec:per-head}--\ref{sec:critical-temp}. A separate entropy condition, ~\ref{cond:E}, appears later in Proposition~\ref{proposition:multiH-entropy} and is unrelated. The different labels ($\tau$ vs.\ E) are intentional.
\end{remark}

\begin{definition}[Radial-tangential decomposition]\label{def:radial-tangential}
  For each head $h$ and token $i$, define:
  \begin{itemize}[nosep]
    \item the \textbf{radial component} $a_i^h = \langle f_i^h, x_i \rangle$ (projection of head $h$'s aggregation onto the token position),
    \item the \textbf{tangential magnitude} $b_i^h = \|P_{x_i}^\perp(f_i^h)\| = \|f_i^h - a_i^h x_i\|$,
    \item the \textbf{alignment fraction} $\rho_i^h = |a_i^h| / \|f_i^h\| = |\cos \theta_i^h|$, where $\theta_i^h$ is the angle between $f_i^h$ and $x_i$.
  \end{itemize}
  By Pythagoras: $(a_i^h)^2 + (b_i^h)^2 = \|f_i^h\|^2$.
  \paragraph{Degenerate case.} If $\|f_i^h\| = 0$, then $a_i^h = b_i^h = \rho_i^h = 0$. For generic initial data on $S^{d-1}$ with $d \geq 2$, $\|f_i^h\| > 0$ for all $i, h$. Indeed, $f_i^h = 0$ requires every token to lie in $\ker(M_h)$, which is a proper subspace under Condition~\ref{cond:P}. Generic tokens on $S^{d-1}$ do not all lie there, so $\|f_i^h\| > 0$ holds generically and is preserved along the flow. We exclude the degenerate set henceforth.
\end{definition}

\subsection{Standing Assumptions}\label{sec:standing}

Table~\ref{tab:standing} lists the conditions required by each main result. Conditions S and V are required throughout and define the gradient flow structure. All other conditions, including orthogonality (O, P), radial dominance ($\tau$), scalar heads ($M_h = \lambda_h I$), and equiangular/orthogonal token configurations, are imposed only in the theorems that explicitly list them. Results in Sections~\ref{sec:hetero}--\ref{sec:convergence} are restricted to the scalar-head, equiangular-token regime and do not claim generality beyond that setting.

\begin{table}[H]
\centering
\caption{Conditions required by each main result. ``S, V'' are required throughout. ``Scalar'' means $M_h = \lambda_h I_d$; ``equiangular'' means all pairwise inner products $\langle x_i, x_j\rangle$ are equal for $i \neq j$.}
\label{tab:standing}
\renewcommand{\arraystretch}{1.25}
\begin{tabular}{lll}
\toprule
\textbf{Result} & \textbf{Conditions} & \textbf{Regime} \\
\midrule
Thm~\ref{thm:total} (total energy mono.) & S, V & General $M_h$, general tokens \\
Thm~\ref{thm:approx-val} (approx.\ value) & S & General $M_h$, general tokens \\
Thm~\ref{thm:flat-perhead} (flat per-head) & S, V, O, P & General $M_h$ (projection), general tokens \\
Thm~\ref{thm:sphere-perhead} (sphere per-head) & S, V, O, P, $\tau$ & General $M_h$ (projection), general tokens \\
Thm~\ref{thm:critical-temp} (critical $\beta^*$) & S, V, O, P, $\tau$ & \textbf{Scalar heads, orthogonal tokens} \\
Thm~\ref{thm:normalized} (normalized flow) & S, V & General $M_h$, general tokens \\
Thm~\ref{thm:hetero-rates} (hetero.\ rates) & S, V & \textbf{Scalar heads, equiangular tokens} \\
Thm~\ref{thm:relu-softmax} (ReLU vs.\ softmax) & --- & \textbf{Scalar heads, equiangular tokens, linearized} \\
Thm~\ref{thm:entropy} (entropy identity) & S & General $M_h$ (any score-symmetric) \\
Cor~\ref{cor:entropy-sign} (entropy sign) & S & \textbf{Scalar heads, equiangular tokens} \\
Thm~\ref{thm:approx-orth} (approx.\ orth.) & S, V & General $M_h$, general tokens \\
Thm~\ref{thm:convergence-rate} (conv.\ rate) & S, V, O & \textbf{Scalar heads, equiangular tokens} \\
\bottomrule
\end{tabular}
\end{table}

\section{Total Energy Monotonicity}\label{sec:total-energy}

\begin{theorem}[Total Energy Monotonicity]\label{thm:total}
  \emph{(Requires: Conditions~\ref{cond:S} and~\ref{cond:V}; general score matrices and token configurations.)}
  For both flat ($\dot{x}_i = v_i$) and sphere ($\dot{x}_i = P_{x_i}^\perp(v_i)$) dynamics:
  \begin{equation}\label{eq:total-monotone}
    \frac{d\E_{\mathrm{multi}}}{dt} = \frac{1}{n}\sum_{i=1}^{n} \|\dot{x}_i\|^2 \geq 0.
  \end{equation}
\end{theorem}

\begin{proof}
  \textbf{Step 1: Differentiate each $\E^h$ by the chain rule.}
  Fix a head $h$. Differentiating $\E^h = \frac{1}{2\beta n^2}\sum_{i,j} e^{\beta \langle x_i, M_h x_j \rangle}$:
  \[
    \frac{d\E^h}{dt} = \frac{1}{2\beta n^2}\sum_{i,j} e^{\beta \langle x_i, M_h x_j \rangle} \cdot \beta \frac{d}{dt}\langle x_i, M_h x_j \rangle.
  \]
  The $\beta$ cancels with $1/\beta$. Since $M_h$ is constant, $\frac{d}{dt}\langle x_i, M_h x_j \rangle = \langle \dot{x}_i, M_h x_j \rangle + \langle x_i, M_h \dot{x}_j \rangle$, so:
  \begin{equation}\label{eq:dEh-step1}
    \frac{d\E^h}{dt} = \frac{1}{2n^2}\sum_{i,j} e^{\beta \langle x_i, M_h x_j \rangle}\bigl(\langle \dot{x}_i, M_h x_j \rangle + \langle x_i, M_h \dot{x}_j \rangle\bigr).
  \end{equation}

  \textbf{Step 2: Symmetrize using Condition~\ref{cond:S}.}
  Write \eqref{eq:dEh-step1} as $A + B$. In $B$, swap $i \leftrightarrow j$:
  \[
    B = \frac{1}{2n^2}\sum_{i,j} e^{\beta \langle x_j, M_h x_i \rangle}\langle x_j, M_h \dot{x}_i \rangle.
  \]
  By Condition~\ref{cond:S}: $\langle x_j, M_h x_i \rangle = \langle x_i, M_h x_j \rangle$ and $\langle x_j, M_h \dot{x}_i \rangle = \langle \dot{x}_i, M_h x_j \rangle$, so $B = A$ and the $1/2$ cancels:
  \begin{equation}\label{eq:dEh-step2}
    \frac{d\E^h}{dt} = \frac{1}{n^2}\sum_{i} \left\langle \dot{x}_i,\, \underbrace{\sum_{j} e^{\beta \langle x_i, M_h x_j \rangle} M_h x_j}_{= f_i^h} \right\rangle = \frac{1}{n^2}\sum_{i} \langle \dot{x}_i, f_i^h \rangle.
  \end{equation}

  \textbf{Step 3: Sum over heads and apply Condition~\ref{cond:V}.}
  Summing \eqref{eq:dEh-step2} over all $H$ heads:
  \begin{equation}\label{eq:dE-multi-sum}
    \frac{d\E_{\mathrm{multi}}}{dt} = \frac{1}{n^2}\sum_{i} \left\langle \dot{x}_i,\, \sum_{h} f_i^h \right\rangle.
  \end{equation}
  By \eqref{eq:velocity}, $\sum_h f_i^h = nv_i$, so:
  \begin{equation}\label{eq:dE-multi-vel}
    \frac{d\E_{\mathrm{multi}}}{dt} = \frac{1}{n}\sum_{i} \langle \dot{x}_i, v_i \rangle.
  \end{equation}

  \textbf{Step 4a: Flat case ($\dot{x}_i = v_i$).}
  \[
    \frac{d\E_{\mathrm{multi}}}{dt} = \frac{1}{n}\sum_{i} \|v_i\|^2 = \frac{1}{n}\sum_{i} \|\dot{x}_i\|^2 \geq 0.
  \]

  \textbf{Step 4b: Sphere case ($\dot{x}_i = P_{x_i}^\perp(v_i)$).}
  Substituting into \eqref{eq:dE-multi-vel} and applying Identity~\ref{id:K} with $u = v_i$:
  \[
    \frac{d\E_{\mathrm{multi}}}{dt} = \frac{1}{n}\sum_{i} \langle P_{x_i}^\perp(v_i), v_i \rangle = \frac{1}{n}\sum_{i}\|P_{x_i}^\perp(v_i)\|^2 = \frac{1}{n}\sum_{i} \|\dot{x}_i\|^2 \geq 0.
  \]
\end{proof}

\begin{theorem}[Approximate Value Alignment]\label{thm:approx-val}
  \emph{(Requires: Condition~\ref{cond:S}; builds on the chain-rule and symmetrization steps of Theorem~\ref{thm:total}; general score matrices and token configurations.)}
  Suppose the value projections satisfy $W^{V,h} = M_h + E_h$ with $\|E_h\|_{\mathrm{op}} \leq \varepsilon$ for all $h$. Define the perturbed aggregation $\tilde{f}_i^h = \sum_j e^{\beta\langle x_i, M_h x_j\rangle} W^{V,h} x_j$, the perturbed velocity $\tilde{v}_i = \frac{1}{n}\sum_h \tilde{f}_i^h$, and sphere dynamics $\dot{x}_i = P_{x_i}^\perp(\tilde{v}_i)$. Then:
  \begin{equation}\label{eq:approx-val}
    \frac{d\E_{\mathrm{multi}}}{dt}
    \geq
    \frac{1}{n}\sum_{i}\|\dot{x}_i\|^2
    - \frac{\varepsilon}{n^2}\sum_{i}\|\dot{x}_i\|\sum_h\|\tilde{g}_i^h\|,
  \end{equation}
  where $\tilde{g}_i^h = \sum_j e^{\beta\langle x_i, M_h x_j\rangle} x_j$ is the unweighted aggregation. In particular, $d\E_{\mathrm{multi}}/dt \geq 0$ whenever
  \[
    \varepsilon \leq \varepsilon^*(t) := \frac{\sum_i \|\dot{x}_i\|^2}{\frac{1}{n}\sum_i \|\dot{x}_i\|\sum_h \|\tilde{g}_i^h\|}.
  \]
\end{theorem}

\begin{proof}
  \textbf{Step 1: Decompose the velocity.}
  Write $\tilde{f}_i^h = f_i^h + e_i^h$ where $e_i^h = E_h\tilde{g}_i^h$, so $\tilde{v}_i = v_i + \frac{1}{n}\sum_h e_i^h$.

  \textbf{Step 2: Energy derivative.}
  Steps~1--2 of Theorem~\ref{thm:total} depend only on the score-matrix structure of $\E^h$, which is unchanged by the value perturbation. They give $d\E_{\mathrm{multi}}/dt = \frac{1}{n}\sum_i \langle \dot{x}_i, v_i\rangle$.

  \textbf{Step 3: Substitute the perturbed velocity.}
  Write $v_i = \tilde{v}_i - \frac{1}{n}\sum_h e_i^h$:
  \[
    \frac{d\E_{\mathrm{multi}}}{dt}
    = \frac{1}{n}\sum_i \langle \dot{x}_i, \tilde{v}_i\rangle
      - \frac{1}{n^2}\sum_i\sum_h \langle \dot{x}_i, e_i^h\rangle.
  \]
  Identity~\ref{id:K} gives $\langle \dot{x}_i, \tilde{v}_i\rangle = \|\dot{x}_i\|^2$. Cauchy--Schwarz and $\|e_i^h\| \leq \varepsilon\|\tilde{g}_i^h\|$ bound the second term.

  \textbf{Step 4: Combine.}
  \[
    \frac{d\E_{\mathrm{multi}}}{dt}
    \geq
    \frac{1}{n}\sum_i \|\dot{x}_i\|^2
    - \frac{\varepsilon}{n^2}\sum_i \|\dot{x}_i\|\sum_h\|\tilde{g}_i^h\|,
  \]
  which is non-negative when $\varepsilon \leq \varepsilon^*(t)$.
\end{proof}

\begin{remark}[Interpretation of $\varepsilon^*(t)$]\label{rem:approx-val-interp}
  The threshold $\varepsilon^*(t)$ is the ratio of total kinetic energy to the perturbation's leverage. It stays bounded away from zero as long as the dynamics have not halted. This parallels Theorem~\ref{thm:approx-orth}: both give monotonicity-under-small-perturbation results with explicit data-dependent thresholds. Together, they show the Lyapunov structure survives the two main idealizations in the model.
\end{remark}

\begin{remark}[Sphere is slower than flat]\label{rem:sphere-slower}
  $\|\dot{x}_i\|^2_{\mathrm{sphere}} = \|P_{x_i}^\perp(v_i)\|^2 = \|v_i\|^2 - \langle v_i, x_i\rangle^2 \leq \|v_i\|^2 = \|\dot{x}_i\|^2_{\mathrm{flat}}$, with equality if and only if $v_i \perp x_i$. Both rates are non-negative; the sphere version is smaller because the tangential projection removes the radial part.
\end{remark}

\begin{remark}[Wasserstein gradient flow]\label{rem:wasserstein}
  The identity $d\E_{\mathrm{multi}}/dt = \frac{1}{n}\sum_i \|\dot{x}_i\|^2$ identifies the dynamics as Wasserstein gradient ascent of $\F[\mu] = \sum_h \frac{1}{2\beta}\iint e^{\beta \langle x, M_h y \rangle}\,d\mu(x)\,d\mu(y)$ in the sense of \cite{AGS05}. The first variation, evaluated at $\mu = \frac{1}{n}\sum_j \delta_{x_j}$, gives $v_i$. This is a restatement of Theorem~\ref{thm:total}, not an independent result.
\end{remark}

\section{Per-Head Energy Monotonicity}\label{sec:per-head}

The total energy grows monotonically, but individual heads can fail to do so:

\begin{figure}[H]
\centering
\includegraphics[width=\textwidth]{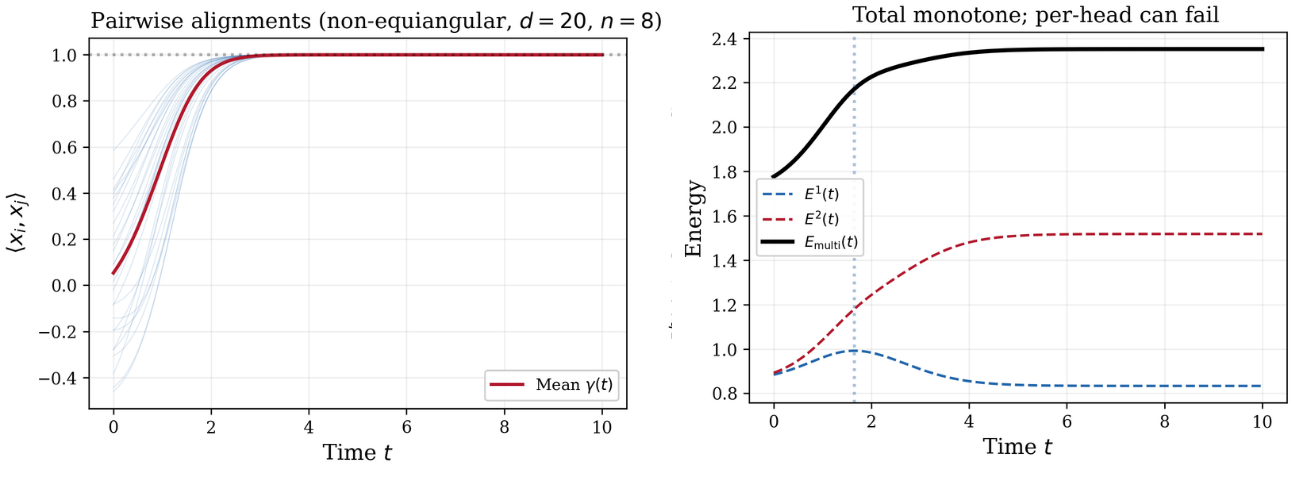}
\caption{
Alignment and energy dynamics in the non-equiangular regime.
\emph{Left:} Pairwise alignments $\langle x_i, x_j \rangle$ over time for $d=20$, $n=8$. Individual trajectories (light blue) converge rapidly, while the mean alignment $\gamma(t)$ (red) grows monotonically.
\emph{Right:} Per-head energies $E^1(t)$ (blue, dashed) and $E^2(t)$ (red, dashed), alongside $E_{\mathrm{multi}}(t)$ (black). Individual heads can be non-monotone while the total is monotone, illustrating that the Lyapunov structure is global. The vertical dashed line marks the transition where head-wise dynamics diverge.
}
\label{fig:alignment-energy-non-equiangular}
\end{figure}

The figure shows that individual heads can temporarily decrease while the sum grows. Monotonicity is a global property, not a head-by-head one. We now make this precise.

\subsection{Flat Space}

\begin{theorem}[Flat Per-Head Monotonicity]\label{thm:flat-perhead}
  \emph{(Requires: Conditions~\ref{cond:S},~\ref{cond:V},~\ref{cond:O},~\ref{cond:P}; general token configurations; builds on Step~2 of Theorem~\ref{thm:total}.)}
  Assume $M_h \succeq 0$ for each $h$. Then:
  \begin{equation}\label{eq:flat-perhead}
    \frac{d\E^h}{dt} = \frac{1}{n^3}\sum_{i=1}^{n} \|f_i^h\|^2 \geq 0.
  \end{equation}
\end{theorem}

\begin{proof}
  From \eqref{eq:dEh-step2}, $\frac{d\E^h}{dt} = \frac{1}{n^2}\sum_i \langle \dot{x}_i, f_i^h \rangle$. In flat space, $\dot{x}_i = v_i = \frac{1}{n}\sum_{h'} f_i^{h'}$. Under Condition~\ref{cond:P}, $M_{h'} f_i^{h'} = f_i^{h'}$, so substituting:
  \begin{equation}\label{eq:flat-perhead-expand}
    \frac{d\E^h}{dt} = \frac{1}{n^3}\sum_{i}\sum_{h'} (f_i^{h'})^\top M_{h'} f_i^h.
  \end{equation}

  \emph{Cross terms ($h' \neq h$):} Since $f_i^h = M_h w_i^h$ where $w_i^h = \sum_j e^{\beta \langle x_i, M_h x_j \rangle} x_j$, Condition~\ref{cond:O} gives $M_{h'} f_i^h = M_{h'} M_h w_i^h = 0$. So $(f_i^{h'})^\top M_{h'} f_i^h = 0$.

  \emph{Self term ($h' = h$):} $(f_i^h)^\top f_i^h = \|f_i^h\|^2 \geq 0$.

  Thus $\frac{d\E^h}{dt} = \frac{1}{n^3}\sum_i \|f_i^h\|^2 \geq 0$.
\end{proof}

\subsection{Sphere --- The Radial Shadow Obstruction}

On the sphere, Condition~\ref{cond:O} is not enough. Even with $M_{h'} M_h = 0$ exactly, scalar projections of cross-head outputs onto $x_i$ survive and interfere. We call these \emph{radial shadows}. The theorem below identifies this precisely.

\begin{theorem}[Sphere Per-Head Monotonicity]\label{thm:sphere-perhead}
  \emph{(Requires: Conditions~\ref{cond:S},~\ref{cond:V},~\ref{cond:O},~\ref{cond:P},~\ref{cond:tau}; general token configurations; extends Theorem~\ref{thm:flat-perhead} to sphere dynamics.)}
  \begin{equation}\label{eq:sphere-perhead}
    \frac{d\E^h}{dt}
    = \frac{1}{n^3}\sum_{i}(b_i^h)^2
      - \frac{1}{n^3}\sum_{i} a_i^h \sum_{h' \neq h} a_i^{h'}
    \geq 0.
  \end{equation}
  The bound holds on $[0, T_{\mathrm{crit}})$, the interval during which Condition~\ref{cond:tau} is satisfied. Since $\rho_i^h \to 1$ as $\gamma \to 1$ and $c^*(H) < 1$ for all $H \geq 2$, Condition~\ref{cond:tau} eventually fails; see Remark~\ref{rem:Tcrit} and Open Problem~1.
\end{theorem}

\begin{proof}
  Four steps: reduce to an inner product against $P_{x_i}^\perp(f_i^h)$, evaluate cross-head terms, evaluate the self term, bound the interference.

  \medskip\noindent\textbf{Step 1: Reduction via projection transfer.}
  From \eqref{eq:dEh-step2} and $\dot{x}_i = P_{x_i}^\perp(v_i)$ with $v_i = \frac{1}{n}\sum_{h'} M_{h'} f_i^{h'}$ (using Condition~\ref{cond:P}):
  \[
    \frac{d\E^h}{dt} = \frac{1}{n^3}\sum_{i} \left\langle P_{x_i}^\perp\!\left(\sum_{h'} M_{h'} f_i^{h'}\right), f_i^h \right\rangle.
  \]
  Apply Identity~\ref{id:SA} to transfer the projection to the second argument:
  \begin{equation}\label{eq:sphere-perhead-step1}
    \frac{d\E^h}{dt} = \frac{1}{n^3}\sum_{i}\sum_{h'} \langle M_{h'} f_i^{h'},\, P_{x_i}^\perp(f_i^h) \rangle.
  \end{equation}
  Expand $P_{x_i}^\perp(f_i^h) = f_i^h - a_i^h x_i$ (using the radial component $a_i^h = \langle f_i^h, x_i\rangle$):
  \begin{equation}\label{eq:sphere-perhead-step2}
    \frac{d\E^h}{dt} = \frac{1}{n^3}\sum_{i}\sum_{h'}\bigl[\underbrace{\langle M_{h'} f_i^{h'}, f_i^h \rangle}_{\text{(A)}} - a_i^h \underbrace{\langle M_{h'} f_i^{h'}, x_i \rangle}_{\text{(B)}}\bigr].
  \end{equation}

  \medskip\noindent\textbf{Step 2: Cross-head terms ($h' \neq h$).}

  \emph{Term (A):} $(f_i^{h'})^\top M_{h'} f_i^h = (f_i^{h'})^\top M_{h'} M_h w_i^h = 0$ by Condition~\ref{cond:O}.

  \emph{Term (B):} By Condition~\ref{cond:S} and $f_i^{h'} \in \Img(M_{h'})$: $\langle M_{h'} f_i^{h'}, x_i \rangle = \langle f_i^{h'}, x_i \rangle = a_i^{h'}$.

  So each cross-head summand contributes $0 - a_i^h \cdot a_i^{h'} = -a_i^h a_i^{h'}$.

  \medskip\noindent\textbf{Step 3: Self term ($h' = h$).}
  By Condition~\ref{cond:P}, $M_h f_i^h = f_i^h$, so $\langle M_h f_i^h,\, P_{x_i}^\perp(f_i^h) \rangle = \langle f_i^h, P_{x_i}^\perp(f_i^h) \rangle = \|P_{x_i}^\perp(f_i^h)\|^2 = (b_i^h)^2$ by Identity~\ref{id:K}.

  \medskip\noindent\textbf{Assembling.}
  \begin{equation}\label{eq:sphere-assembled}
    \frac{d\E^h}{dt}
    = \frac{1}{n^3}\sum_{i}(b_i^h)^2
      - \frac{1}{n^3}\sum_{i} a_i^h \sum_{h' \neq h} a_i^{h'}.
  \end{equation}

  \medskip\noindent\textbf{Step 4: Bounding the interference.}
  By the triangle inequality and Condition~\ref{cond:tau}:
  \[
    \frac{1}{n^3}\sum_{i} |a_i^h|\sum_{h' \neq h} |a_i^{h'}|
    \leq \frac{1}{n^3}\sum_{i} |a_i^h| \cdot \frac{(b_i^h)^2}{\|f_i^h\|}
    = \frac{1}{n^3}\sum_{i} \rho_i^h (b_i^h)^2
    \leq \frac{1}{n^3}\sum_{i} (b_i^h)^2.
  \]
  So $\frac{d\E^h}{dt} \geq \frac{1}{n^3}\sum_i (1-\rho_i^h)(b_i^h)^2 \geq 0$.
\end{proof}

\begin{remark}[The radial shadow obstruction]\label{rem:radial-shadow}
  Condition~\ref{cond:O} kills the operator-level term $M_{h'} M_h = 0$. However, it is unable to kill $a_i^{h'} = \langle f_i^{h'}, x_i \rangle$, since $x_i$ is not constrained to $\Img(M_{h'})^\perp$. These radial shadows, projections of each head's output onto the token position, survive regardless of how orthogonal the head subspaces are. Condition~\ref{cond:tau} is the precise requirement that the sum of these shadows does not exceed the head's own tangential power.
\end{remark}

\section{Critical Temperature Threshold}\label{sec:critical-temp}

\subsection{The Critical Alignment Fraction $c^*(H)$}

In the uniform regime where all heads have alignment $\rho_i^h = \rho$, Condition~\ref{cond:tau} simplifies to $(H-1)\rho \leq 1 - \rho^2$, i.e.,
\begin{equation}\label{eq:rho-quadratic}
  \rho^2 + (H-1)\rho - 1 \leq 0.
\end{equation}
The positive root is
\begin{equation}\label{eq:cstar}
  c^*(H) = \frac{\sqrt{(H-1)^2 + 4} - (H-1)}{2}.
\end{equation}
For $H = 2$: $c^*(2) = (\sqrt{5}-1)/2 = 1/\varphi$. For $H = 3$: $c^*(3) = \sqrt{2}-1$.

\begin{theorem}[Critical Temperature]\label{thm:critical-temp}
  \emph{(Requires: Conditions~\ref{cond:S},~\ref{cond:V},~\ref{cond:O},~\ref{cond:P},~\ref{cond:tau}; builds on the radial shadow analysis of Theorem~\ref{thm:sphere-perhead}; \textbf{restricted to scalar heads ($M_h = \alpha I_d$) and orthogonal tokens}.)}

  For $M_h = \alpha I_d$ ($\alpha > 0$) and $n$ mutually orthogonal tokens, Condition~\ref{cond:tau} holds if and only if $\beta \leq \beta^*$, where
  \begin{equation}\label{eq:beta-star}
    \beta^* = \frac{1}{2\alpha}\ln\frac{c^*(H)^2(n-1)}{1 - c^*(H)^2}, \qquad \text{requiring } n > 1/c^*(H)^2.
  \end{equation}
\end{theorem}

\begin{proof}
  \textbf{Step 1: Compute the aggregation.}
  For $M_h = \alpha I_d$ and $\langle x_i, x_j\rangle = \delta_{ij}$:
  \[
    f_i^h = \alpha\Bigl(e^{\alpha\beta} x_i + \sum_{j \neq i} x_j\Bigr).
  \]

  \textbf{Step 2: Compute the radial component.}
  $a_i^h = \langle f_i^h, x_i\rangle = \alpha e^{\alpha\beta}$, since the self-pairing gives 1 and the off-diagonal terms vanish by orthogonality.

  \textbf{Step 3: Compute $\|f_i^h\|$.}
  $\|f_i^h\|^2 = \alpha^2(e^{2\alpha\beta} + n-1)$, so $\|f_i^h\| = \alpha\sqrt{e^{2\alpha\beta}+n-1}$.

  \textbf{Step 4: Compute $\rho_i^h$ (the alignment fraction).}
  \[
    \rho_i^h = \frac{|a_i^h|}{\|f_i^h\|} = \frac{e^{\alpha\beta}}{\sqrt{e^{2\alpha\beta}+n-1}}.
  \]

  \textbf{Step 5: Solve $\rho_i^h = c^*(H)$.}
  Set $s = e^{\alpha\beta}$ (the self-attention Boltzmann weight). Then $s/\sqrt{s^2+n-1} = c^*(H)$. Squaring: $s^2 = c^*(H)^2(n-1)/(1-c^*(H)^2)$. Taking logarithms gives \eqref{eq:beta-star}.

\end{proof}

\begin{remark}[Scope of scalar and equiangular assumptions]
  \label{rem:scalar-scope}
  Theorems~\ref{thm:critical-temp}, \ref{thm:hetero-rates}, and~\ref{thm:convergence-rate} all rely on scalar score matrices $M_h = \lambda_h I_d$ and/or equiangular/orthogonal token configurations. Scalar heads collapse the geometry to a single parameter; equiangular tokens reduce the full $n$-token ODE on $S^{d-1}$ to a scalar ODE in $\gamma$. Both are strong idealizations. Extending the closed-form $\beta^*$, the super-additivity threshold $\lambda_c$, and the exponential convergence rate to general $M_h$ and non-equiangular tokens is an important open direction.
\end{remark}


\section{Normalized Flow and Log-Partition Obstruction}\label{sec:normalized}

\begin{theorem}[Normalized Flow]\label{thm:normalized}
  \emph{(Requires: Conditions~\ref{cond:S} and~\ref{cond:V}; general score matrices and token configurations; builds on Theorem~\ref{thm:total}.)}
  Let $Z_i^h = \sum_j e^{\beta\langle x_i, M_h x_j\rangle}$ (the normalizing constant for head $h$ at token $i$) and $Z_i = \sum_h Z_i^h$. Under normalized dynamics $\dot{x}_i = P_{x_i}^\perp(v_i/Z_i)$:
  \begin{equation}\label{eq:normalized-lyapunov}
    \frac{d\E_{\mathrm{multi}}}{dt} = \frac{1}{n}\sum_{i} Z_i \|\dot{x}_i\|^2 \geq 0.
  \end{equation}
  The log-partition energy $G = \frac{1}{\beta n}\sum_i \log Z_i$ generally does not have a monotone derivative.
\end{theorem}

\begin{proof}
  \emph{Monotonicity.} Steps 1--2 of Theorem~\ref{thm:total} are unchanged, so \eqref{eq:dE-multi-vel} holds. With $\dot{x}_i = P_{x_i}^\perp(v_i/Z_i)$, apply Identity~\ref{id:SA} then Identity~\ref{id:K}:
  \[
    \langle P_{x_i}^\perp(v_i/Z_i), v_i\rangle = \frac{1}{Z_i}\|P_{x_i}^\perp(v_i)\|^2 = Z_i\|\dot{x}_i\|^2.
  \]
  Summing gives \eqref{eq:normalized-lyapunov}.

  \emph{Failure of $G$.} Differentiating $G$ and swapping summation indices shows the derivative involves row-weighted aggregates $\hat{f}_i = \sum_h\sum_j A_{ij}^h M_h x_j$ and column-weighted aggregates $\tilde{f}_i = \sum_h\sum_j A_{ji}^h M_h x_j$ where $A_{ij}^h = e^{\beta\langle x_i, M_h x_j\rangle}/Z_i$. Since $A_{ij}^h \neq A_{ji}^h$ whenever $Z_i \neq Z_j$, the sign of $dG/dt$ is indefinite.
\end{proof}


\section{Heterogeneous Head Convergence Rates}\label{sec:hetero}

\begin{theorem}[Heterogeneous Scalar Head Rates]\label{thm:hetero-rates}
  \emph{(Requires: Conditions~\ref{cond:S},~\ref{cond:V}; \textbf{restricted to scalar heads $M_h = \lambda_h I_d$ and equiangular token configurations}; see Remark~\ref{rem:scalar-scope}.)}

  For $M_h = \lambda_h I_d$ ($\lambda_h > 0$ distinct) with $\Lambda = \sum_h \lambda_h$, in the equiangular reduction $\gamma = \langle x_i, x_j\rangle$ ($i \neq j$) with $\varepsilon = 1-\gamma$:
  \begin{enumerate}[label=(\alph*),nosep]
    \item \textbf{Late-time:} $\dot{\varepsilon} = -2\Lambda\varepsilon + O(\varepsilon^2)$, giving $\varepsilon(t) \sim Ce^{-2\Lambda t}$.
    \item \textbf{Early-time:} $g^h(0) = 2\lambda_h/(e^{\lambda_h\beta}+n-1)$, maximized at $\lambda^*\beta = 1 + W((n-1)/e)$ where $W$ is the Lambert $W$-function.
    \item \textbf{Super-additivity:} The rate function $\phi(\lambda) = 2\lambda/(e^{\lambda\beta}+n-1)$ has inflection point $\lambda_c > \lambda^*$ satisfying
\begin{equation}\label{eq:lambda_c}
  (\lambda_c\beta - 2)\,e^{\lambda_c\beta} = (\lambda_c\beta + 2)(n-1),
\end{equation}
and is strictly concave on $(0,\lambda_c)$ and strictly convex on $(\lambda_c,\infty)$. When the mean head strength $\bar\lambda = \Lambda/H$ satisfies $\bar\lambda > \lambda_c$,
\[
  \sum_h \phi(\lambda_h) > H\phi(\bar\lambda)
\]
for any distinct $\lambda_1,\ldots,\lambda_H$ with $\sum_h \lambda_h = \Lambda$. Heterogeneous heads cluster faster than equal heads at the same total strength.

\begin{remark}[$\lambda^*$ is not the inflection point]
$\lambda^*$ maximizes $\phi$ ($\phi'(\lambda^*) = 0$), while $\lambda_c$ is where $\phi'' = 0$. They cannot coincide for $\lambda > 0$. Numerically, $\lambda_c/\lambda^* \approx 1.7$ for small $n$. The condition $\bar\lambda > \lambda_c$ is sufficient but not necessary.
\end{remark}
  \end{enumerate}
\end{theorem}

\begin{proof}
  \textbf{Step 1: Equiangular ODE.}
  With $M_h = \lambda_h I$ and all $\langle x_i, x_j\rangle = \gamma$ for $i \neq j$, permutation equivariance reduces the $n$-token dynamics to a scalar ODE for $\gamma$ (following the symmetry reduction of \cite{GLPR23,GLPR25} applied to each head independently). The combined driving force is $\dot{\gamma} = \sum_h g^h(\gamma)$ where
  \begin{equation}\label{eq:equiangular-ode}
    g^h(\gamma) = \frac{2\lambda_h e^{\lambda_h\beta\gamma}(1-\gamma)(1+(n-1)\gamma)}{e^{\lambda_h\beta}+(n-1)e^{\lambda_h\beta\gamma}}.
  \end{equation}

  \begin{remark}[Equiangular reduction]\label{rem:equiangular}
  The reduction requires all pairwise inner products to be equal at $t=0$ and to remain so. Permutation equivariance of the ODE guarantees this is preserved if it holds initially. For non-equiangular data, $\gamma$ becomes a vector of pairwise angles and the scalar ODE does not apply; Theorems~\ref{thm:hetero-rates} and~\ref{thm:convergence-rate} hold only on this symmetric submanifold.
  \end{remark}

  \textbf{Step 2: Late-time expansion ($\varepsilon \to 0$).}
  Set $\gamma = 1-\varepsilon$ (so $\varepsilon$ measures distance to full clustering). To leading order, $g^h(1-\varepsilon) \approx 2\lambda_h\varepsilon$. Summing: $\dot{\varepsilon} = -2\Lambda\varepsilon + O(\varepsilon^2)$, with solution $\varepsilon(t) \sim Ce^{-2\Lambda t}$.

  \textbf{Step 3: Early-time rate at $\gamma = 0$.}
  $g^h(0) = 2\lambda_h/(e^{\lambda_h\beta}+n-1)$ by direct substitution.

  \textbf{Step 4: Maximize over $\lambda_h$.}
  Setting $\partial g^h(0)/\partial\lambda_h = 0$ gives $(1-\lambda^*\beta)e^{\lambda^*\beta} = -(n-1)$.

  \textbf{Step 5: Lambert $W$ form.}
  Let $u = \lambda^*\beta-1$. Then $ue^u = (n-1)/e$, so $u = W((n-1)/e)$ and $\lambda^*\beta = 1+W((n-1)/e)$.

  \textbf{Step 6: Concavity--convexity structure.}
  Setting $E = e^{\lambda\beta}$ and $\sigma = E+n-1$:
  \[
    \phi'(\lambda) = \frac{2((1-\lambda\beta)E+n-1)}{\sigma^2}, \qquad
    \phi''(\lambda) = \frac{2\beta E[(\lambda\beta-2)E-(\lambda\beta+2)(n-1)]}{\sigma^3}.
  \]
  At $\lambda^*$, $\phi'(\lambda^*) = 0$ implies $(1-\lambda^*\beta)E^* = -(n-1)$; substituting into the numerator of $\phi''$ gives $\phi''(\lambda^*) = -2(\lambda^*\beta)^2\beta (E^*)^2/\sigma^{*3} < 0$, so $\lambda^*$ is a strict maximum. The numerator of $\phi''$ is negative at $\lambda^*$ and grows without bound as $\lambda \to \infty$, so there is a unique $\lambda_c > \lambda^*$ where it vanishes, giving the stated concavity/convexity.

  \textbf{Step 7: Super-additivity via Jensen.}
  For $\bar\lambda > \lambda_c$, all $\lambda_h$ near $\bar\lambda$ lie in the convex region. Jensen's inequality gives $\phi(\bar\lambda) \leq \frac{1}{H}\sum_h\phi(\lambda_h)$, with strict inequality when the $\lambda_h$ are not all equal. When some $\lambda_h$ fall below $\lambda_c$, a Taylor expansion to third order and the bound $|\phi'''(\lambda)| \leq 6\beta^2\phi(\lambda)/\lambda$ give the quantitative sufficient condition
\[
  \max_h|D_h| < \frac{3\phi''(\bar\lambda)\min_h\lambda_h}{\beta^2\max_h\phi(\lambda_h)},
\]
where $D_h = \lambda_h - \bar\lambda$, under which super-additivity still holds.
\end{proof}

\begin{figure}[H]
\centering
\includegraphics[width=\textwidth]{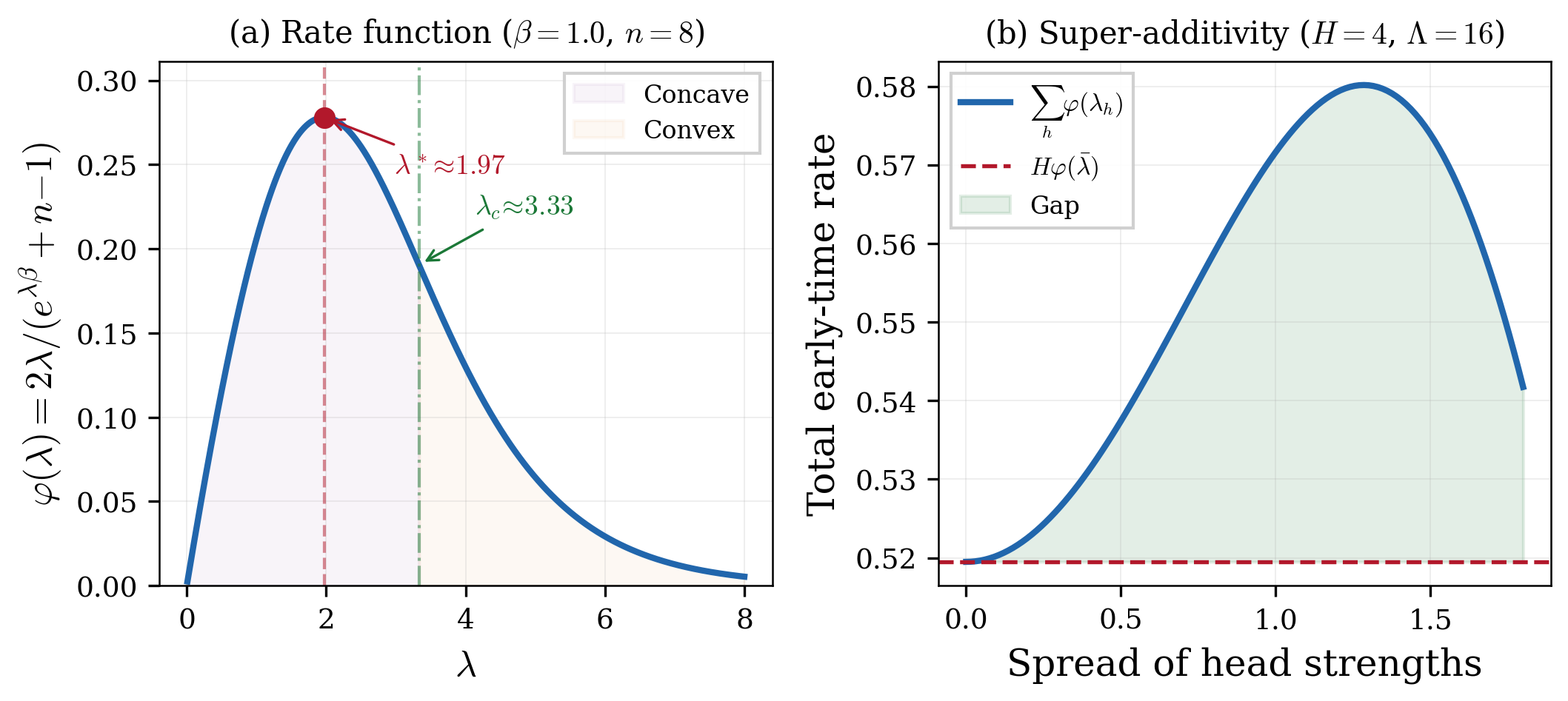}
\caption{Rate function and super-additivity (Theorem~\ref{thm:hetero-rates}). \emph{(a)}~$\phi(\lambda)=2\lambda/(e^{\lambda\beta}+n-1)$ with maximum $\lambda^*\approx1.97$ and inflection $\lambda_c\approx3.33$. \emph{(b)}~For $H=4$ heads with $\Lambda=16$ in the convex regime ($\bar\lambda=4>\lambda_c$), heterogeneous strengths strictly exceed equal strengths. Parameters: $\beta=1.0$, $n=8$.}
\label{fig:superadditivity}
\end{figure}


\section{ReLU vs.\ Softmax Clustering Time}\label{sec:relu-softmax}

\begin{theorem}[Clustering Time Separation --- Linearized Regime]\label{thm:relu-softmax}
  \emph{(Requires: \textbf{scalar head $M = \lambda I$, equiangular-token reduction, linearized ODE near $\gamma=0$}; builds on Theorem~\ref{thm:hetero-rates}. The full nonlinear regime is deferred to Open Problem~\ref{prob:relu}.)}

  For random tokens on $S^{d-1}$ with $\gamma_0 = O(1/\sqrt{d})$, the following separation holds in the linearized regime near $\gamma = 0$:
  \begin{enumerate}[label=(\alph*),nosep]
    \item \textbf{ReLU} ($\sigma(t) = \max(0,t)$): $\dot\gamma = (2\lambda^2/n)\gamma + O(\gamma^2)$ near $\gamma=0$, giving $T_{\mathrm{ReLU}} = O(n\log d)$.
    \item \textbf{Softmax}: $\dot\gamma|_{\gamma=0} = 2\lambda/(e^{\lambda\beta}+n-1) > 0$, a positive constant independent of $d$, giving $T_{\mathrm{softmax}} = O(n)$.
  \end{enumerate}
  A complete proof for the nonlinear regime requires handling the non-smooth boundary $\{\gamma=0\}$ and the regime $\gamma = \Omega(1)$; see Open Problem~\ref{prob:relu}.
\end{theorem}

\begin{proof}
  \textbf{Step 1: Initialization.}
  For $x_i$ uniformly random on $S^{d-1}$, $\langle x_i,x_j\rangle \approx \mathcal{N}(0,1/d)$ by the central limit theorem, so $\gamma_0 = O(1/\sqrt{d})$ with high probability.

  \textbf{Step 2: Softmax near $\gamma=0$.}
  From Theorem~\ref{thm:hetero-rates}, $g^h(0) = c_0 = 2\lambda/(e^{\lambda\beta}+n-1) > 0$ independent of $d$. The ODE $\dot\gamma \approx c_0 > 0$ gives $\gamma(t) \approx \gamma_0 + c_0 t$, so $T_{\mathrm{softmax}} = O(1/c_0) = O(n)$.

  \textbf{Step 3: ReLU equiangular ODE.}
  With $M = \lambda I$, the ReLU aggregation gives velocity $v_i = (\lambda^2/n)[x_i + \gamma\sum_{j\neq i} x_j]$ when $\gamma > 0$. Computing $\langle\dot{x}_i, x_j\rangle$ and factoring:
  \begin{equation}\label{eq:relu-ode}
    \dot\gamma = \frac{2\lambda^2\gamma}{n}(1-\gamma)(1+(n-1)\gamma).
  \end{equation}
  At $\gamma=0$: $\dot\gamma = 0$ since $\mathrm{ReLU}(0) = 0$. Near $\gamma=0^+$: $\dot\gamma \approx (2\lambda^2/n)\gamma$.

  \textbf{Step 4: ReLU clustering time.}
  From the linearized ODE: $\gamma(t) = \gamma_0 e^{2\lambda^2 t/n} = (C/\sqrt{d})\,e^{2\lambda^2 t/n}$. Setting $\gamma(T) = \Omega(1)$: $T = O(n\log d)$.
\end{proof}

\begin{remark}[Hybrid architectures]\label{rem:hybrid}
  This is an informal comparison; a full version requires the nonlinear extension of Theorem~\ref{thm:relu-softmax}. At early times, ReLU is completely silent at $\gamma = 0$ while softmax drives everything. At late times, softmax heads carry a suppression factor $e^{-\lambda_h\beta}$ from over-concentration \cite{ZLT23}, while ReLU heads carry none, so even weak ReLU heads dominate late-time convergence at large $\beta$. This is consistent with the empirical finding that ReLU attention can match or approach softmax performance \cite{Wor23,She23}.
\end{remark}

\begin{figure}[H]
\centering
\includegraphics[width=\textwidth]{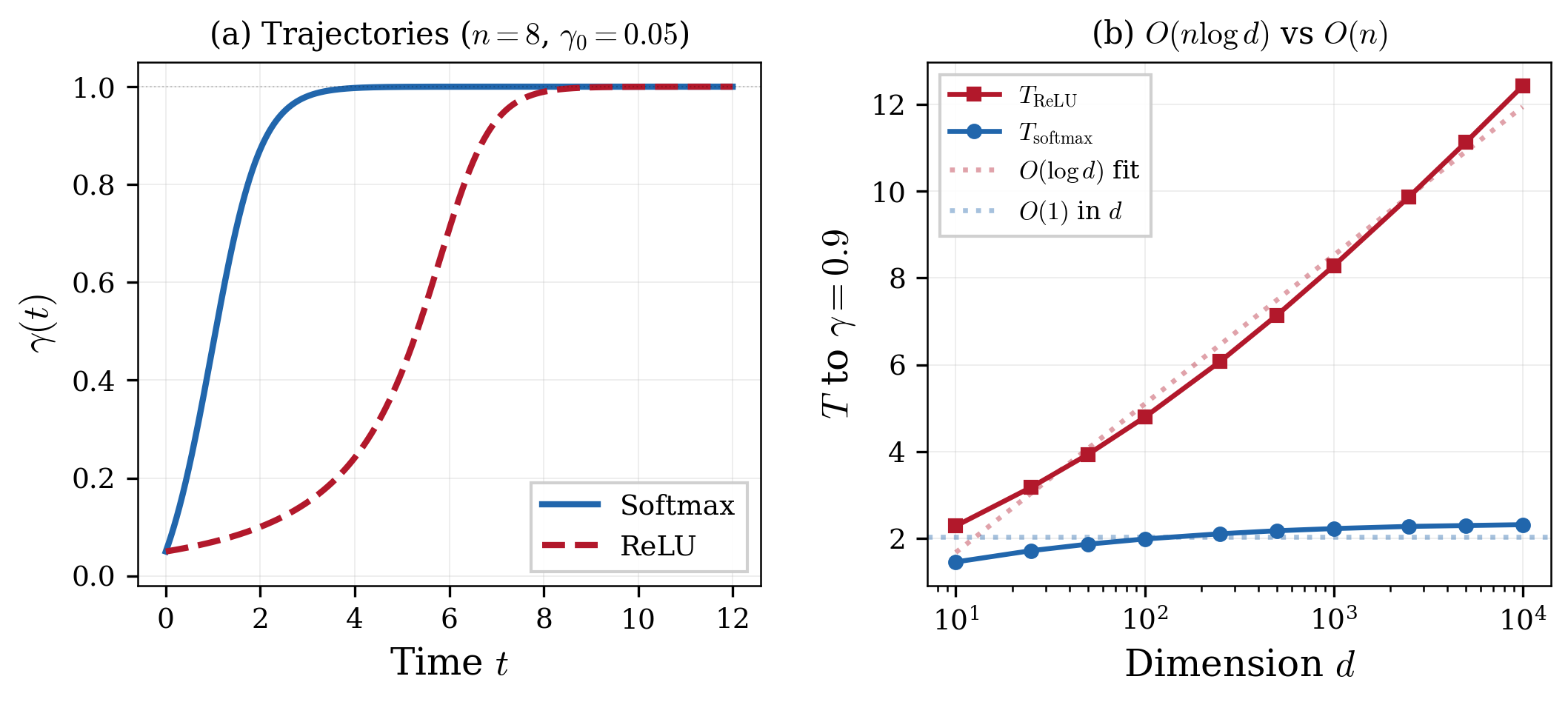}
\caption{ReLU vs.\ softmax clustering time (Theorem~\ref{thm:relu-softmax}). \emph{(a)}~Softmax reaches $\gamma\approx1$ roughly four times faster for $n=8$, $\gamma_0=0.05$. \emph{(b)}~$T_{\mathrm{ReLU}}$ grows as $O(\log d)$ while $T_{\mathrm{softmax}}$ is $O(1)$ in $d$, confirming the separation. Parameters: $\lambda=1.0$, $\beta=1.0$, $n=8$.}
\label{fig:relu-softmax}
\end{figure}


\section{Attention Entropy Production Identity}\label{sec:entropy}

The normalized attention distribution for head $h$ at token $i$ is
\[
  p_{ij}^h = \frac{e^{\beta\langle x_i, M_h x_j\rangle}}{Z_i^h}, \quad Z_i^h = \sum_k e^{\beta\langle x_i, M_h x_k\rangle}, \quad H_i^h = -\sum_j p_{ij}^h \log p_{ij}^h.
\]

\begin{theorem}[Entropy Production Identity]\label{thm:entropy}
  \emph{(Requires: Condition~\ref{cond:S} only; valid for any score-symmetric attention and any sphere dynamics. The sign of the covariance is established in Corollary~\ref{cor:entropy-sign} for the scalar-head equiangular case.)}

  Write $s_j = \langle x_i, M_h x_j\rangle$ (the score of token $j$ from token $i$ under head $h$) and $\dot{s}_j = \langle\dot{x}_i, M_h x_j\rangle + \langle M_h x_i, \dot{x}_j\rangle$. Then:
  \begin{equation}\label{eq:entropy-production}
    \frac{dH_i^h}{dt} = -\beta^2\,\Cov_{p_{i\cdot}^h}(s_j,\,\dot{s}_j).
  \end{equation}
  This is an exact identity. The sign of the covariance depends on the dynamics and is established separately.
\end{theorem}

\begin{proof}
  Write $Z = Z_i^h$, $p_j = p_{ij}^h$, $\mathbb{E}_p[\cdot] = \sum_j p_j(\cdot)$.

  \textbf{Step 1: Exponential-family decomposition.}
  Since $\log p_j = \beta s_j - \log Z$:
  \begin{equation}\label{eq:entropy-decomp}
    H_i^h = \log Z - \beta\,\mathbb{E}_p[s].
  \end{equation}

  \textbf{Step 2: Differentiate $\log Z$.}
  \begin{equation}\label{eq:dlogZ}
    \frac{d\log Z}{dt} = \beta\,\mathbb{E}_p[\dot{s}].
  \end{equation}

  \textbf{Step 3: Differentiate $\mathbb{E}_p[s]$.}
  Since $\dot{p}_j = \beta p_j(\dot{s}_j - \mathbb{E}_p[\dot{s}])$:
  \begin{equation}\label{eq:dEps}
    \frac{d}{dt}\mathbb{E}_p[s] = \beta\,\Cov_p(s,\dot{s}) + \mathbb{E}_p[\dot{s}].
  \end{equation}

  \textbf{Step 4: Combine.}
  \[
    \frac{dH_i^h}{dt} = \beta\,\mathbb{E}_p[\dot{s}] - \beta(\beta\,\Cov_p(s,\dot{s}) + \mathbb{E}_p[\dot{s}]) = -\beta^2\,\Cov_p(s,\dot{s}).
  \]

  \textbf{Step 5: Sign of $\Cov_p(s,\dot{s})$.}
  Split $\dot{s}_j = \dot{s}^I_j + \dot{s}^{II}_j$ where $\dot{s}^I_j = \langle\dot{x}_i, Mx_j\rangle$ (token $i$ moving) and $\dot{s}^{II}_j = \langle Mx_i, \dot{x}_j\rangle$ (token $j$ moving). In general, neither $\Cov_p(s,\dot{s}^I)$ nor $\Cov_p(s,\dot{s}^{II})$ vanishes individually. The sign of their sum $\Cov_p(s,\dot{s})$ depends on the dynamics and is established in Corollary~\ref{cor:entropy-sign} for the scalar-head equiangular case via a direct two-group calculation.

  \begin{remark}[Component~I in the scalar equiangular case]\label{rem:component-I}
  When $M = \lambda I$ and tokens are equiangular, $\dot{s}^I_j$ takes only two values (zero for $j = i$ and a common positive constant for $j \neq i$), and $\Cov_p(s,\dot{s}^I)$ is strictly negative, not zero. Both components contribute to the total covariance with the same sign.
  \end{remark}
\end{proof}

\begin{remark}[Information flow interpretation]
  The identity $dH/dt = -\beta^2\Cov_p(s_j,\dot{s}_j)$ has a clean interpretation. Writing $\Cov_p(s,\dot{s}) = \mathbb{E}_p[(s-\mathbb{E}_p[s])\dot{s}]$, the sign is determined by whether above-average score tokens tend to increase or decrease. If high-score tokens increase, attention sharpens and entropy falls. If they decrease, mass spreads and entropy rises. The formula makes this direction computable.
\end{remark}

\begin{corollary}[Entropy monotonicity --- equiangular case]
  \label{cor:entropy-sign}
  \emph{(Requires: Condition~\ref{cond:S}; builds on Theorem~\ref{thm:entropy} (Step~5); \textbf{restricted to scalar head $M = \lambda I_d$ and equiangular configuration}.)}

  For a single head $M = \lambda I_d$ and $\langle x_j,x_k\rangle = \gamma$ ($j \neq k$) with $\gamma \in (0,1)$:
  \[
    \Cov_{p_i}(s_j,\dot{s}_j) \leq 0, \qquad \text{so} \qquad \frac{dH_i}{dt} \geq 0.
  \]
  In the scalar-head equiangular regime, attention entropy is monotonically non-decreasing for all $n \geq 2$.
\end{corollary}

\begin{proof}
  We compute $\Cov_p(s_j,\dot{s}_j)$ directly via the two-group structure of the equiangular scalar-head case.

  \textbf{Step 1: Score and velocity structure.}
  Scores: $s_i = \lambda$ (self) and $s_j = \lambda\gamma$ for $j \neq i$. The softmax has two mass types: $p_i = e^{\beta\lambda}/Z$ and $p_j = e^{\beta\lambda\gamma}/Z$ for $j \neq i$.

  \textbf{Step 2: Full score velocities.}
  Recall $\dot{s}_j = \dot{s}^I_j + \dot{s}^{II}_j$ where $\dot{s}^I_j = \langle\dot{x}_i, \lambda x_j\rangle$ and $\dot{s}^{II}_j = \lambda\langle x_i,\dot{x}_j\rangle$.

  \emph{Component~II:} For $j = i$: $\dot{s}^{II}_i = \lambda\langle x_i,\dot{x}_i\rangle = 0$ since $\dot{x}_i \perp x_i$ on $S^{d-1}$.
  For $j \neq i$: by the equiangular reduction, $\frac{d}{dt}\langle x_i,x_j\rangle = \dot\gamma = 2\langle x_i,\dot{x}_j\rangle$, so $\dot{s}^{II}_j = \lambda\dot\gamma/2 =: c > 0$.

  \emph{Component~I:} For $j = i$: $\dot{s}^I_i = \lambda\langle\dot{x}_i, x_i\rangle = 0$ since $\dot{x}_i \perp x_i$.
  For $j \neq i$: by equiangular symmetry, $\langle\dot{x}_i, x_j\rangle$ is the same constant $c'$ for all $j \neq i$. Since $\dot\gamma = \langle\dot{x}_i, x_j\rangle + \langle x_i, \dot{x}_j\rangle = c' + c/\lambda$, we have $c' = \dot\gamma - c/\lambda$. So $\dot{s}^I_j = \lambda c'$ for $j \neq i$.

  \emph{Total:} $\dot{s}_i = 0$ and $\dot{s}_j = \lambda c' + c$ for $j \neq i$.  Both values are positive constants during clustering.

  \textbf{Step 3: Two-group covariance.}
  The joint distribution takes two values: $(s_i, \dot{s}_i) = (\lambda, 0)$ with weight $p_i$ and $(s_j, \dot{s}_j) = (\lambda\gamma, \lambda c' + c)$ with total weight $1 - p_i$. The two-group covariance formula gives:
  \[
    \Cov_p(s,\dot{s}) = p_i(1-p_i)\cdot\lambda(1-\gamma)\cdot(-(\lambda c' + c)) < 0.
  \]
  Each factor is positive for $\gamma \in (0,1)$ and $n \geq 2$ (since $\lambda c' + c = \lambda\dot\gamma > 0$ during clustering), so the covariance is strictly negative and $dH_i/dt = -\beta^2\Cov_p(s,\dot{s}) > 0$.
\end{proof}

\begin{remark}[Physical interpretation]\label{rem:entropy-direction}
  The anti-monotone pairing is forced by the sphere: the highest score (self-score $\lambda$) is paired with zero velocity ($\dot{s}^{II}_i = 0$) because tangential motion cannot change $\langle x_i, x_i\rangle = 1$. Lower cross-scores are paired with positive velocity. As tokens cluster, all pairwise scores equalize toward $\lambda$, softmax converges to uniform, and entropy rises to $\log n$.
\end{remark}

\begin{remark}[Entropy sign beyond the scalar case]
  \label{rem:entropy-general}
  The cancellation $\dot{s}^{II}_i = 0$ relies on $Mx_i \propto x_i$, which holds when $M = \lambda I$. For general symmetric $M \succeq 0$, this fails and $\Cov_p(s_j,\dot{s}_j)$ need not be non-positive. Numerical experiments show it can become positive along the flow. The identity \eqref{eq:entropy-production} remains exact; only the sign is regime-dependent.
\end{remark}

\begin{lemma}[Cross-head covariance]\label{lem:crosscov}
  \emph{(Requires: Conditions~\ref{cond:S},~\ref{cond:V},~\ref{cond:O},~\ref{cond:P}; extends the Component~I analysis of Theorem~\ref{thm:entropy} to the multi-head setting.)}
  For head $h$ and token $i$, let $A_h^{(i)} = \sum_{h' \neq h} a_i^{h'}$ (total radial shadow from other heads). Then
\[
  \Cov_{p_i^h}(s_j,\,\dot{s}^{\,\mathrm{I}}_j) = -\frac{A_h^{(i)}}{n}\,\Var_{p_i^h}(s_j).
\]
\end{lemma}

\begin{proof}
  Decompose $\dot{x}_i = \dot{x}_i^{(h)} + \dot{x}_i^{(-h)}$ where $\dot{x}_i^{(h)} = \frac{1}{n}P_{x_i}^\perp(f_i^h)$ and $\dot{x}_i^{(-h)} = \frac{1}{n}\sum_{h'\neq h} P_{x_i}^\perp(f_i^{h'})$.

  For the cross part, Identity~\ref{id:SA} and Conditions~\ref{cond:P},~\ref{cond:O} give $\langle P_{x_i}^\perp(f_i^{h'}), M_h x_j\rangle = -s_j a_i^{h'}$. The self part $\dot{x}_i^{(h)}$ contributes $\Cov_{p_i^h}(s_j, \langle\dot{x}_i^{(h)}, M_h x_j\rangle)$, which need not vanish in general (see Remark~\ref{rem:component-I}). The stated formula captures the cross-head contribution: summing over $h' \neq h$ and taking the covariance gives $\Cov_{p_i^h}(s_j, -(s_j/n)A_h^{(i)}) = -(A_h^{(i)}/n)\Var_{p_i^h}(s_j)$.
\end{proof}

\begin{proposition}[Multi-head entropy decomposition]\label{proposition:multiH-entropy}
  \emph{(Requires: Conditions~\ref{cond:S},~\ref{cond:V},~\ref{cond:O},~\ref{cond:P}; builds on Theorem~\ref{thm:entropy} and Lemma~\ref{lem:crosscov}.)}
  For $H \geq 2$:
\[
  \frac{dH_i^h}{dt} = -\frac{\beta^2}{n}\,\Cov_{\mathrm{single}} + \frac{\beta^2 A_h^{(i)}}{n}\,\Var(s_j) + \frac{\beta^2}{n}\sum_{h'\neq h} \Cov_{p_i^h}(s_j,\,a_j^{h'} s_j).
\]
The right-hand side is $\leq 0$ whenever:
  \begin{equation}\tag{Condition~E}\label{cond:E}
    \Cov_{\mathrm{single}}(s_j, \dot{s}^{\,\mathrm{II},h}_j)
    \geq
    A_h^{(i)}\,\Var(s_j)
    + \sum_{h'\neq h}
      \Cov_{p_i^h}(s_j,\,a_j^{h'} s_j).
  \end{equation}
\end{proposition}

\begin{remark}[When Condition~E holds]\label{rem:condE}
  In Phase~1 (pre-clustering), $A_h^{(i)} \approx 0$ by approximate cross-head orthogonality and the cross-covariance terms are small. Condition~E then reduces to $\Cov_{\mathrm{single}} \leq 0$, which holds in the equiangular scalar-head case by Corollary~\ref{cor:entropy-sign} but not for general $M$ (Remark~\ref{rem:entropy-general}). In Phase~2, $\dot{s}_j \to 0$ and both sides vanish. The intermediate regime remains the main open case.
\end{remark}

\subsection{Two-Phase Structure of Entropy Dynamics}

\begin{remark}[Phase 1: Pre-clustering]\label{rem:phase1}
  When $\gamma \ll 1$, the self-score $s_i = \lambda$ dominates the cross-scores $s_j = \lambda\gamma$. Cross-score velocities $\dot{s}_j$ are large and positive as tokens approach each other, while $\dot{s}_i = 0$ by the sphere constraint. This anti-monotone pairing drives $dH_i^h/dt \gg 0$: entropy rises quickly as attention spreads from peaked toward uniform.
\end{remark}

\begin{remark}[Phase 2: Near full clustering]\label{rem:phase2}
  As $\gamma \to 1$, all pairwise scores equalize, $p_{ij}^h \to 1/n$, and $\dot{s}_j \to 0$. So $\Cov_p(s,\dot{s}) \to 0$ and $dH_i^h/dt \to 0$: entropy stabilizes at $\log n$.
\end{remark}

\begin{remark}[Why entropy increases, not decreases]\label{rem:naive-confusion}
  It may seem like tokens clustering together should sharpen attention and decrease entropy. That would be right if clustering were selective, with some tokens approaching while others stayed distant. However, the gradient flow produces a single-cluster collapse where all tokens approach equally. All pairwise scores equalize, softmax converges to uniform, and entropy increases. The formula $dH_i^h/dt = -\beta^2\Cov_p(s,\dot{s})$ captures this exactly. Entropy stabilizes only when all scores equalize and the dynamics halt.
\end{remark}

\begin{figure}[H]
\centering
\includegraphics[width=\textwidth]{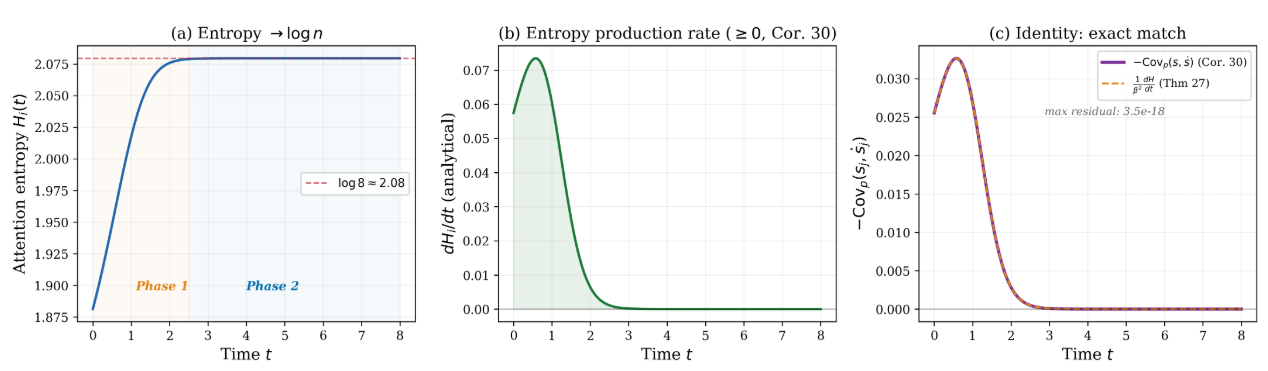}
\caption{Entropy production identity and two-phase structure (Theorem~\ref{thm:entropy}, Corollary~\ref{cor:entropy-sign}). \emph{(a)}~$H_i(t)$ increases toward $\log n$ in two phases. \emph{(b)}~$dH_i/dt \geq 0$ peaks during pre-clustering and vanishes at equilibrium. \emph{(c)}~The identity $dH/dt = -\beta^2\Cov_p(s_j,\dot{s}_j)$ holds to machine precision (max residual $3\times10^{-18}$). Parameters: $M=\lambda I$, $\lambda=1.0$, $\beta=1.5$, $n=8$, $\gamma_0=0.05$.}
\label{fig:entropy}
\end{figure}

\section{Approximate Head Orthogonality}\label{sec:approx-orth}

Condition~\ref{cond:O} is never exactly satisfied in trained transformers. Empirical principal angles between Q/K subspaces are $70$--$85^\circ$, not $90^\circ$. We relax it to $\|M_{h'}M_h\|_\op \leq \delta$.

\begin{theorem}[Approximate Orthogonality]\label{thm:approx-orth}
  \emph{(Requires: Conditions~\ref{cond:S},~\ref{cond:V}, and $\|M_{h'}M_h\|_{\mathrm{op}} \leq \delta$ for all $h' \neq h$; flat dynamics; general token configurations; extends Theorem~\ref{thm:flat-perhead} to approximately orthogonal heads.)}

  In the flat case:
  \begin{equation}\label{eq:approx-orth}
    \frac{d\E^h}{dt} \geq \frac{1}{n^3}\!\left(\sum_{i} (f_i^h)^\top M_h f_i^h - \delta' \sum_{i} \|f_i^h\| \sum_{h' \neq h} \|f_i^{h'}\|\right) \geq 0
  \end{equation}
  whenever $\delta' \leq \delta^*(t)$, where
  \[
    \delta' = \frac{\delta}{\sigma_{\min}(M_h)}, \qquad \delta^*(t) = \frac{\sum_i (f_i^h)^\top M_h f_i^h}{\sum_i \|f_i^h\| \sum_{h' \neq h} \|f_i^{h'}\|},
  \]
  and $\sigma_{\min}(M_h) > 0$ is the smallest nonzero singular value of $M_h$.
\end{theorem}

\begin{proof}
  From \eqref{eq:flat-perhead-expand}, $\frac{d\E^h}{dt} = \frac{1}{n^3}\sum_i\sum_{h'} (f_i^{h'})^\top M_{h'} f_i^h$.

  \emph{Self term:} $(f_i^h)^\top M_h f_i^h \geq 0$.

  \emph{Cross terms:} Write $w_i^h = \sum_j e^{\beta\langle x_i, M_h x_j\rangle} x_j$, so $f_i^h = M_h w_i^{h,\parallel}$ where $w_i^{h,\parallel} \in \mathrm{row}(M_h)$. Then:
  \[
    |(f_i^{h'})^\top M_{h'} f_i^h| \leq \|f_i^{h'}\| \cdot \delta \cdot \|w_i^{h,\parallel}\| \leq \delta' \|f_i^{h'}\|\|f_i^h\|,
  \]
  using $\|w_i^{h,\parallel}\| \leq \|f_i^h\|/\sigma_{\min}(M_h)$ and $\delta' = \delta/\sigma_{\min}(M_h)$. Summing and applying the triangle inequality gives \eqref{eq:approx-orth}.
\end{proof}

\section{Sharp Per-Head Convergence Rate Near Equilibrium}\label{sec:convergence}

\begin{theorem}[Per-Head Exponential Convergence]
  \label{thm:convergence-rate}
  \emph{(Requires: Conditions~\ref{cond:S},~\ref{cond:V},~\ref{cond:O}; builds on the late-time expansion of Theorem~\ref{thm:hetero-rates}(a); \textbf{restricted to scalar heads $M_h = \lambda_h I_d$ and equiangular tokens}.)}

  With $\Lambda = \sum_h \lambda_h$ and $\Delta\E^h = \E^{h*} - \E^h(t)$ near the clustered equilibrium ($\varepsilon = 1-\gamma \to 0$):
  \begin{equation}\label{eq:convergence}
    \Delta\E^h(t) \leq C_h \exp(-2\Lambda t),
  \end{equation}
  where $C_h = \frac{(n-1)\lambda_h}{2n}e^{\beta\lambda_h}\varepsilon_0$.
\end{theorem}

\begin{proof}
  \textbf{Step 1: Energy gap in terms of $\varepsilon$.}
  For diagonal score $\lambda_h$ and off-diagonal score $\lambda_h(1-\varepsilon)$, expanding around $\varepsilon = 0$:
  \[
    \Delta\E^h \approx \frac{(n-1)\lambda_h}{2n}e^{\beta\lambda_h}\,\varepsilon.
  \]

  \textbf{Step 2: ODE for $\varepsilon$.}
  From Theorem~\ref{thm:hetero-rates}(a): $\varepsilon(t) \approx \varepsilon_0 e^{-2\Lambda t}$.

  \textbf{Step 3: Combine.}
  $\Delta\E^h(t) \leq C_h e^{-2\Lambda t}$ with $C_h = \frac{(n-1)\lambda_h}{2n}e^{\beta\lambda_h}\varepsilon_0$.

\begin{remark}[Normalized vs.\ unnormalized near equilibrium]
  \label{rem:2overn}
  Bound \eqref{eq:convergence} comes from normalized dynamics \eqref{eq:equiangular-ode}. For unnormalized dynamics \eqref{eq:velocity}, the factor $(1+(n-1)\gamma) \to n$ near equilibrium exactly cancels the $1/n$ prefactor, giving rate $2\sum_h \lambda_h e^{\beta\lambda_h}$ instead of $2\Lambda$. The cancellation is geometric, not due to normalization.
\end{remark}
\end{proof}

\begin{figure}[H]
\centering
\includegraphics[width=\textwidth]{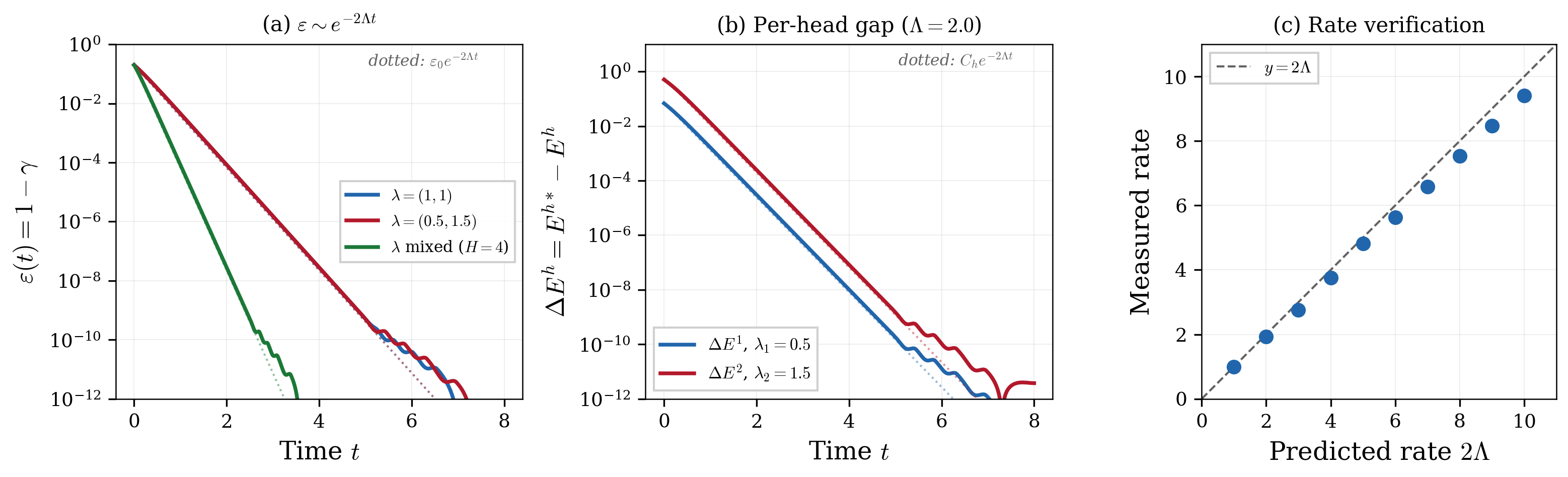}
\caption{Exponential convergence near equilibrium (Theorem~\ref{thm:convergence-rate}). \emph{(a)}~$\varepsilon(t)=1-\gamma(t)$ decays at rate $2\Lambda$ (dotted: $\varepsilon_0 e^{-2\Lambda t}$). \emph{(b)}~Per-head energy gap $\Delta E^h$ for $\lambda_1=0.5$, $\lambda_2=1.5$ at shared rate $2\Lambda=4$. \emph{(c)}~Measured rates vs.\ $2\Lambda$ across a range of total strengths. All: $n=8$, $\beta=1.0$, $\gamma_0=0.8$.}
\label{fig:convergence}
\end{figure}

\section{Conclusion}\label{sec:conclusion}

We have built a rigorous framework for multi-head self-attention as a gradient flow on the unit sphere, extending the single-head theory of \cite{GLPR23,GLPR25}.

The central finding is the radial shadow: a cross-head interference term in the per-head energy derivative that survives even when head subspaces are exactly orthogonal. Total energy $\E_{\mathrm{multi}}$ is monotone unconditionally (Theorem~\ref{thm:total}). Per-head monotonicity requires the Radial Dominance condition (Condition~\ref{cond:tau}), which holds precisely when $\beta \leq \beta^*$, a threshold we compute in closed form via the golden ratio and Lambert $W$-function in the scalar-head, orthogonal-token regime (Theorem~\ref{thm:critical-temp}).

The remaining quantitative results are all in the scalar-head, equiangular-token regime (see Remark~\ref{rem:scalar-scope} and Table~\ref{tab:standing}). Theorem~\ref{thm:hetero-rates} shows that spreading head strengths strictly improves early-time clustering when the mean strength is in the convex regime of the rate function. Theorem~\ref{thm:relu-softmax} gives a rigorous basis, in the linearized regime, for the observation that ReLU and softmax have complementary strengths: softmax drives clustering from $\gamma = 0$ while ReLU is silent there, but ReLU dominates at late times when softmax over-concentrates. The entropy identity (Theorem~\ref{thm:entropy}) and its sign result (Corollary~\ref{cor:entropy-sign}) make the two-phase dynamics precise: entropy increases during pre-clustering as scores equalize and attention spreads, then stabilizes at $\log n$.

The main open questions are the behavior after $T_{\mathrm{crit}}$ (Open Problem~\ref{prob:trajectory-invariance}), the structure of critical points of the multi-head flow (Open Problem~\ref{prob:hardproblem}), and the sign of $\Cov_p(s_j,\dot{s}_j)$ for general score matrices, where entropy monotonicity need not hold (Remark~\ref{rem:entropy-general}).

\section{Open Problems}\label{sec:open}

\begin{openproblem}[Trajectory-invariance of Condition~\ref{cond:tau}]\label{prob:trajectory-invariance}
  Per-head monotonicity holds only on $[0,T_{\mathrm{crit}})$. Since $\rho_i^h \to 1 > c^*(H)$ as $\gamma \to 1$, Condition~\ref{cond:tau} eventually fails for any $H \geq 2$. Is there a natural class of initial conditions for which the trajectory stays in the $\tau$-satisfying region long enough to conclude clustering? The entropy production rate $\beta^2\Cov_p(s_j,\dot{s}_j)$ controls how quickly $\rho_i^h$ grows, and could yield a lower bound on $T_{\mathrm{crit}}$, though bounding $\dot\rho_i^h$ in terms of $dH_i^h/dt$ directly remains open. This is also related to the metastability studied in \cite{GKPR24} and \cite{BPA25}.
\begin{remark}[$T_{\mathrm{crit}}$ bound]\label{rem:Tcrit}
$T_{\mathrm{crit}} = \inf\{t : \rho_i^h(t) = c^*(H)\}$. Since $\rho_i^h$ satisfies a nonlinear ODE coupled to the full token dynamics, a closed-form lower bound requires controlling $\dot\rho_i^h$ independently of the entropy. The entropy plateau (Phase~2) and the $\rho = c^*(H)$ crossing are distinct events that need not coincide. A quantitative $T_{\mathrm{crit}}$ bound remains open.
\end{remark}
\end{openproblem}

\begin{openproblem}[Critical points of the multi-head flow]\label{prob:hardproblem}
  The Lyapunov structure shows $\E_{\mathrm{multi}}$ increases but does not show that trajectories converge to full clusters. This requires showing all stable critical points are complete clusters. For single-head flows, \cite{GLPR25} does this via a cone-collapse argument; whether it extends to the combined velocity field $\frac{1}{n}\sum_h f_i^h$ is open.
\end{openproblem}

\begin{openproblem}[Per-head Wasserstein structure]
  No single positive-definite metric $G_i$ satisfies $G_i\dot{x}_i = (1/n^2)u_i^h$ for all $h$ simultaneously. A head-indexed family of metrics $\{G^h\}$ under which each $\F^h$ is a gradient flow is being investigated via the Wasserstein--Fisher--Rao framework \cite{CPR25}.
\end{openproblem}

\begin{openproblem}[Normalized flow gradient structure]
  Theorem~\ref{thm:normalized} gives a Lyapunov function but not a gradient flow. Row-normalizing $v_i$ by $Z_i$ introduces an asymmetry whenever $Z_i \neq Z_j$, so $dG/dt$ has indefinite sign.
\end{openproblem}

\begin{openproblem}[Full nonlinear ReLU clustering]\label{prob:relu}
  Theorem~\ref{thm:relu-softmax} establishes the $O(n\log d)$ vs.\ $O(n)$ separation in the linearized regime under scalar-head equiangular assumptions. A full proof additionally requires handling (a)~the non-smooth boundary $\{\langle x_i,x_j\rangle = 0\}$, where the ODE must be interpreted via Dini derivatives or regularization, and (b)~the nonlinear regime $\gamma = \Omega(1)$, where a comparison argument with the full ODE~\eqref{eq:relu-ode} is needed.
\end{openproblem}


\end{document}